\documentclass[english]{article}
\usepackage[T1]{fontenc}
\usepackage[latin9]{inputenc}
\usepackage{geometry}
\usepackage{color}
\usepackage{babel}
\usepackage{float}
\usepackage{mathtools}
\usepackage{amsmath}
\usepackage{amssymb}
\usepackage[unicode=true,pdfusetitle,
 bookmarks=true,bookmarksnumbered=false,bookmarksopen=false,
 breaklinks=false,pdfborder={0 0 0},pdfborderstyle={},backref=false,colorlinks=true]
 {hyperref}
\hypersetup{
 citecolor = blue}

\makeatletter

\floatstyle{ruled}
\newfloat{algorithm}{tbp}{loa}
\providecommand{\algorithmname}{Algorithm}
\floatname{algorithm}{\protect\algorithmname}

\usepackage{babel}

\usepackage{algorithmic}

\usepackage{amsthm}

\theoremstyle{plain}

\newtheorem{lemma}{\textbf{Lemma}}
\newtheorem{theorem}{\textbf{Theorem}}

\newtheorem{definition}{\textbf{Definition}}

\theoremstyle{definition}

\usepackage{color}
\definecolor{cm}{RGB}{0,0,200}

\definecolor{yy}{RGB}{148,0,211}

\newcommand{\NEgap}{\mathrm{NE\text{-}gap}}

\makeatother

\begin{document}
\title{$O(T^{-1})$ Convergence of Optimistic-Follow-the-Regularized-Leader\\
in Two-Player Zero-Sum Markov Games }
\author{{Yuepeng Yang\thanks{Department of Statistics, University of Chicago; Email: \texttt{\{yuepengyang,
congm\}@uchicago.edu}}} \and{Cong Ma\footnotemark[1]}}

\maketitle
\begin{abstract}
We prove that optimistic-follow-the-regularized-leader (OFTRL), together
with smooth value updates, finds an $O(T^{-1})$-approximate Nash
equilibrium in $T$ iterations for two-player zero-sum Markov games
with full information. This improves the $\tilde{O}(T^{-5/6})$ convergence
rate recently shown in the paper~\cite{zhang2022policy}. The refined
analysis hinges on two essential ingredients. First, the sum of the
regrets of the two players, though not necessarily non-negative as
in normal-form games, is approximately non-negative in Markov games.
This property allows us to bound the second-order path lengths of
the learning dynamics. Second, we prove a tighter algebraic inequality
regarding the weights deployed by OFTRL that shaves an extra $\log T$
factor. This crucial improvement enables the inductive analysis that
leads to the final $O(T^{-1})$ rate. 
\end{abstract}

\section{Introduction}

Multi-agent reinforcement learning (MARL)~\cite{busoniu2008comprehensive,zhang2021multi}
models sequential decision-making problems in which multiple agents/players
interact with each other in a shared environment. MARL has recently
achieved tremendous success in playing games~\cite{vinyals2019grandmaster,berner2019dota,brown2019superhuman},
which, consequently, has spurred a growing body of work on MARL; see~\cite{yang2020overview}
for a recent overview. 

A widely adopted mathematical model for MARL is the so-called Markov
games~\cite{shapley1953stochastic,littman1994markov}, which combines
normal-form games~\cite{nash1951non} with Markov decision processes~\cite{puterman2014markov}.
In a nutshell, a Markov game starts with a certain state, followed
by actions taken by the players. The players then receive their respective
payoffs, as in a normal-form game, and at the same time the system
transits to a new state as in a Markov decision process. The whole
process repeats. As in normal-form games, the goal for each player
is to maximize her own cumulative payoffs. We defer the precise descriptions
of Markov games to Section~\ref{sec:Preliminaries}. 

In the simpler normal-form games, no-regret learning~\cite{cesa2006prediction}
has long been used as an effective method to achieve competence in
the multi-agent environment. Take the two-player zero-sum normal-form
game as an example. It is easy to show that standard no-regret algorithms
such as follow-the-regularized-leader (FTRL) reach an $O(T^{-1/2})$-approximate
Nash equilibrium~\cite{nash1951non} in $T$ iterations. Surprisingly,
the seminal paper~\cite{daskalakis2011near} demonstrates that a
special no-regret algorithm, built upon Nesterov's excessive gap technique~\cite{nesterov2005excessive},
achieves a faster and optimal $\tilde{O}(T^{-1})$ rate of convergence
to the Nash equilibrium. This nice and fast convergence was later
established for optimistic variants of mirror descent~\cite{rakhlin2013optimization}
and FTRL~\cite{syrgkanis2015fast}. Since then, a flurry of research
\cite{chen2020hedging,daskalakis2021near,anagnostides2022near,anagnostides2022uncoupled,farina2022near}
has been conducted around optimistic no-regret learning algorithms
to obtain faster rate of convergence in normal-form games. 

In contrast, research on the fast convergence of optimistic no-regret
learning in Markov games has been scarce. In this paper, we focus
on two-player zero-sum Markov games---arguably the simplest Markov
game. \cite{zhang2022policy} recently initiated the
study of the optimistic-follow-the-regularized-leader (OFTRL) algorithm
in such a setting and proved that OFTRL converges to an $\tilde{O}(T^{-5/6})$-approximate
Nash equilibrium after $T$ iterations. In light of the faster $O(T^{-1})$
convergence of optimistic algorithms in normal-form games, it is natural
to ask 

{ 
\begin{itemize}
\item[] \emph{After $T$ iterations, can OFTRL find an $O(T^{-1})$-approximate
Nash equilibrium in two-player zero-sum Markov games? }
\end{itemize}
}

\noindent In fact, this question has also been raised by \cite{zhang2022policy} in the Discussion section. More promisingly,
they have verified the fast convergence (i.e., $O(T^{-1})$) of OFTRL
in a simple two-stage Markov game; see Fig.~1 therein.

Our main contribution in this work is to answer this question \emph{affirmatively},
through improving the $\tilde{O}(T^{-5/6})$ rate demonstrated in
\cite{zhang2022policy} to the optimal $O(T^{-1})$ rate. The improved
rate for OFTRL arises from two technical contributions. The first
is the approximate non-negativity of the sum of the regrets of the
two players in Markov games. In particular, the sum is lower bounded
by the negative estimation error of the optimal $Q$-function; see
Lemma~\ref{lemma:sum_regret_lb} for the precise statement. This
is in stark contrast to the two-player zero-sum normal-form game~\cite{anagnostides2022last}
and the multi-player general-sum normal-form game~ \cite{anagnostides2022uncoupled},
in which by definition, the sum of the external/swap regrets are non-negative.
This approximate non-negativity proves crucial for us to control the
second-order path length of the learning dynamics induced by OFTRL.
In a different context---time-varying zero-sum normal-form games, 
\cite{zhang2022no} also utilizes a sort of approximate non-negativity of the sum of 
the regrets. 
However, the source of this gap from non-negativity is different: in \cite{zhang2022no} it arises from the time-varying nature 
of the zero-sum game, while in our case with Markov games, it comes from 
the estimation error of the equilibrium pay-off matrix by the algorithm itself. 

Secondly, central to the analysis in finite-horizon Markov decision processes
(and also Markov games) is the induction across the horizon. In our
case, in order to carry out the induction step, we prove a tighter
algebraic inequality related to the weights deployed by OFTRL; see
Lemma~\ref{lemma:our-lemma}. In particular, we shave an extra $\log T$
factor. Surprisingly, this seemingly harmless $\log T$ factor is
the key to enabling the above-mentioned induction analysis, and as
a by-product, removes the extra $\log$ factor in the performance
guarantee of OFTRL.

Note that as an imperfect remedy, \cite{zhang2022policy} proposed a modified OFTRL algorithm 
that achieves $\tilde{O}(T^{-1})$ convergence to Nash equilibrium. However, compared to the vanilla OFTRL algorithm considered herein, 
the modified version tracks two $Q$-functions, adopts a different $Q$-function update procedure that can be more costly in certain scenarios, and more importantly diverges from the general policy optimization framework proposed in \cite{zhang2022policy}. Our work bridges these gaps by establishing the fast 
convergence for the vanilla OFTRL. 

Another line of algorithms used for solving Nash equilibrium is based on dynamic programming \cite{perolat2015approximate,zhang2022policy,cen2021fast}. 
Unlike the single-loop structure of OFTRL, the dynamic programming approach requires a nested loop, with the outer-loop iterating over the horizons 
and the inner-loops solving a sub-game through iterations. 
This requires more tuning parameters, one set for each subproblem/layer. Such kind of extra tuning was documented in \cite{cen2021fast}.
The nested nature of dynamic programming also demands one to predetermine a precision $\epsilon$ and estimate the sub-game at each horizon to precision $\epsilon/H$.
This is less convenient in practice compared to a single-loop algorithm like the OFTRL we study, where such predetermined precision is not necessary. 
Another recent paper \cite{cen2022faster} also discusses the advantages of single-loop algorithms over those with nested loops.

\subsection{Related work}

\paragraph*{Optimistic no-regret learning in games. }

Our work is mostly related to the line of work on proving fast convergence
of optimistic no-regret algorithms in various forms of games. \cite{daskalakis2011near} provide the first fast algorithm
that reaches a Nash equilibrium at an $\tilde{O}(T^{-1})$ rate in
two-player zero-sum normal-form games. Later, with the same setup,
\cite{rakhlin2013optimization} prove a similar
fast convergence for optimistic mirror descent (OMD). \cite{syrgkanis2015fast} extend
the results to multi-player general-sum normal-form games. In addition,
Syrgkanis et al.~show that when all the players adopt optimistic
algorithms, their \emph{individual} regret is at most $O(T^{-3/4})$.
This is further improved to $O(T^{-5/6})$ in the special two-player
zero-sum case~\cite{chen2020hedging}. More recently, via a detailed
analysis of higher-order smoothness, \cite{daskalakis2021near,anagnostides2022near}
manage to improve the individual regret guarantee of optimistic hedge
to $\tilde{O}(T^{-1})$ in multi-player general-sum normal-form games,
matching the result in the two-player case. A similar result is shown
by \cite{anagnostides2022uncoupled} with a different analysis using
self-concordant barriers as the regularizer. 

Several attempts have been made to extend the results on optimistic
no-regret learning in normal-form games to Markov games. \cite{wei2021last}
design a decentralized algorithm based on optimistic gradient descent~/~ascent
that converges to a Nash equilibrium at an $\tilde{O}(T^{-1/2})$
rate. Closest to us is the work by \cite{zhang2022policy}
which shows an $\tilde{O}(T^{-5/6})$ convergence of OFTRL to the
Nash equilibrium in two-player zero-sum Markov games and an $\tilde{O}(T^{-3/4})$
convergence to a coarse correlated equilibrium in multi-player general-sum
Markov games. Most recently, \cite{erez2022regret} prove
an $O(T^{-1/4})$ individual regret for OMD in multi-player general-sum
Markov games.

\paragraph*{Two-player zero-sum Markov games. }

Our work also fits into the study of two-player zero-sum Markov games~\cite{shapley1953stochastic,littman1994markov}.
Various algorithms~\cite{hu2003nash,littman1994markov,zhao2021provably,cen2021fast}
have been proposed in the full information setting, where one assumes
the players have access to the \emph{exact} state-action value functions.
In particular, \cite{zhao2021provably,cen2021fast} use optimistic approaches for normal-form games as subroutines
to extend the $\tilde{O}(T^{-1})$ convergence rates to two-player zero-sum Markov games. 
In particular, they provide last iterate convergence guarantees as well. 
However, in doing so, their algorithms require one to approximately solve a normal-form game in each iteration. 

In the bandit setting, \cite{bai2020provable,xie2020learning,bai2020near,liu2021sharp,zhang2020model}
study the sample complexity of two-player zero-sum Markov games. In
addition, \cite{sidford2020solving,jia2019feature,zhang2020model,li2022minimax}
investigate the sample complexity under a generative model where one
can query the Markov game at arbitrary states and actions. Last but not
least, recently two-player zero-sum Markov games have been studied
in the offline setting~\cite{cui2022offline,yan2022model}, where
the learner is given a set of historical data, and cannot interact
with Markov games further.

\section{Preliminaries \label{sec:Preliminaries}}

This section provides the necessary background on Markov games and
optimistic-follow-the-regularized-leader (OFTRL). 

\paragraph{Two-player zero-sum Markov games. }

Denote by $\mathcal{MG}(H,\mathcal{S},\mathcal{A},\mathcal{B},\mathbb{P},r)$
a finite-horizon time-inhomogeneous two-player zero-sum Markov game, with
$H$ the horizon, $\mathcal{S}$ the state space, $\mathcal{A}$ (resp.~$\mathcal{B}$)
the action space for the max-player (resp.~min-player), $\mathbb{P}=\left\{ \mathbb{P}_{h}\right\} _{h\in[H]}$
the transition probabilities, and $r=\left\{ r_{h}\right\} _{h\in[H]}$
the reward function. We assume state space $\mathcal{S}$ and action spaces $\mathcal{A},\mathcal{B}$ to be finite 
and have size $S,A,B$, respectively, and $r_{h}$ takes value in $[0,1]$.
Without loss of generality, we assume that the game starts at a fixed
state $s_{1}\in\mathcal{S}$. Then at each step $h$, both players
observe the current state $s_{h}\in\mathcal{S}$. The max-player picks
an action $a_{h}\in\mathcal{A}$ and the min-player picks an action
$b_{h}\in\mathcal{B}$ simultaneously. Then the max-player (resp.~min-player)
receives the reward $r_{h}(s_{h},a_{h},b_{h})$ (resp.~$-r_{h}(s_{h},a_{h},b_{h})$),
and the game transits to step $h+1$ with the next state $s_{h+1}$
sampled from $\mathbb{P}_{h}(\cdot\mid s_{h},a_{h},b_{h})$. The game
ends after $H$ steps. The goal for the max-player is to maximize
her total reward while the min-player seeks to minimize the total
reward obtained by the max-player. 

\paragraph{Markov policies and value functions.}

Let $\mu=\left\{ \mu_{h}\right\} _{h\in[H]}$ be the Markov policy
for the max-player, where $\mu_{h}(\cdot\mid s)\in\Delta_{\mathcal{A}}$
is the distribution of actions the max-player picks when seeing state
$s$ at step $h$. Here, $\Delta_{\mathcal{X}}$ denotes the set of
all probability distributions on the space $\mathcal{X}$. Similarly,
the min-player is equipped with a Markov policy $\nu=\left\{ \nu_{h}\right\} _{h\in[H]}$.
We define the value function of the policy pair $(\mu,\nu)$ at step
$h$ to be 
\[
V_{h}^{\mu,\nu}(s)\coloneqq\mathbb{E}_{\mu,\nu}\left[\sum_{i=h}^{H}r(s_{i},a_{i},b_{i})\mid s_{h}=s\right],
\]
where the expectation is taken w.r.t.~the policies $\{\mu_{i},\nu_{i}\}_{i\geq h}$
and the state transitions $\{\mathbb{P}_{i}\}_{i\geq h}$. Similarly,
one can define the $Q$-function as
\[
Q_{h}^{\mu,\nu}(s,a,b)\coloneqq\mathbb{E}_{\mu,\nu}\left[\sum_{i=h}^{H}r(s_{i},a_{i},b_{i})\mid s_{h}=s,a_{h}=a,b_{h}=b\right].
\]
In words, both functions represent the expected future rewards received
by the max-player given the current state or state-action pair. 

\paragraph{Best responses and Nash equilibria. }

Fix a Markov policy $\nu$ for the min-player. There exists a Markov
policy $\mu^{\dagger}(\nu)$ (a.k.a.~best response) such that for
any $s\in\mathcal{S}$ and $h\in[H]$,
\[
V_{h}^{\mu^{\dagger}(\nu),\nu}(s)=\sup_{\mu^{\dagger}}V_{h}^{\mu^{\dagger},\nu}(s),
\]
where the supremum is taken over all Markov policies.  To simplify
the notation, we denote $V_{h}^{\dagger,\nu}(s)\coloneqq V_{h}^{\mu^{\dagger}(\nu),\nu}(s)$.
Similarly, we can define $V_{h}^{\mu,\dagger}(s)$. It is known that
a pair $(\mu^{\star},\nu^{\star})$ of Markov policies exists and
$\mu^{\star},\nu^{\star}$ are best responses to the other, i.e.,
$V_{h}^{\mu^{\star},\nu^{\star}}(s)=V_{h}^{\dagger,\nu^{\star}}(s)=V_{h}^{\mu^{\star},\dagger}(s)$
for all $s\in\mathcal{S}$ and $h\in[H]$. Such a pair $(\mu^{\star},\nu^{\star})$
is called a Nash equilibrium (NE). We may denote the value function
and $Q$-function under any Nash equilibrium $(\mu^{\star},\nu^{\star})$
as
\[
V_{h}^{\star}\coloneqq V_{h}^{\mu^{\star},\nu^{\star}},\qquad Q_{h}^{\star}\coloneqq Q_{h}^{\mu^{\star},\nu^{\star}},
\]
which are known to be unique even if there are multiple Nash equilibria~\cite{shapley1953stochastic}.
The goal of learning in two-player zero-sum Markov games is to find
an $\varepsilon$-approximation to the NE defined as follows.

\begin{definition}[$\varepsilon$-approximate Nash equilibrium]Fix
any approximation accuracy $\varepsilon>0$. A pair $(\mu,\nu)$ of
Markov policies is an $\varepsilon$-approximate Nash equilibrium
if 
\begin{equation}
\mathrm{\NEgap(\mu,\nu)}\coloneqq V_{1}^{\dagger,\nu}(s_{1})-V_{1}^{\mu,\dagger}(s_{1})\leq\varepsilon.\label{def:NEGap}
\end{equation}
\end{definition}

\paragraph{An interlude: additional notations. }

Before explaining OFTRL, we introduce some additional notations to
simplify things hereafter. Fix any $h\in[H]$, $s\in\mathcal{S}$.
For any function $Q:\mathcal{S}\times\mathcal{A}\times\mathcal{B}\rightarrow\mathbb{R}$,
we may consider $Q(s,\cdot,\cdot)$ to be an $A\times B$ matrix and
$\mu_{h}(\cdot\mid s),\nu_{h}(\cdot\mid s)$ to be vectors of length
$A$ and $B$, respectively. Then for any policy $(\mu_{h},\nu_{h})$
at horizon $h$ we may define
\begin{align*}
\left[\mu_{h}^{\top}Q\nu_{h}\right](s) & \coloneqq\mathbb{E}_{a\sim\mu_{h}(\cdot\mid s),b\sim\nu_{h}(\cdot\mid s)}[Q(s,a,b)],\\
\left[\mu_{h}^{\top}Q\right](s,\cdot) & \coloneqq\mathbb{E}_{a\sim\mu_{h}(\cdot\mid s)}[Q(s,a,\cdot)],\\
\left[Q\nu_{h}\right](s,\cdot) & \coloneqq\mathbb{E}_{b\sim\nu_{h}(\cdot\mid s)}[Q(s,\cdot,b)].
\end{align*}
The term $\left[\mu_{h}^{\top}Q\nu_{h}\right](s)$ can also be written
in the inner product form $\left\langle \mu_{h},Q\nu_{h}\right\rangle (s)$
or $\left\langle \nu_{h},Q^{\top}\mu_{h}\right\rangle (s)$. It is
easy to check that for fixed $s$ and $h$, the left hand sides of
these definitions are standard matrix operations. In addition, for
any $V:\mathcal{S}\mapsto\mathbb{R}$, we define the shorthand
\[
\left[\mathbb{P}_{h}V\right](s,a,b)\coloneqq\mathbb{E}_{s'\sim\mathbb{P}_{h}(\cdot\mid s,a,b)}[V(s')],
\]
which allows us to rewrite Bellman updates of $V$ and $Q$ as 
\[
V_{h}^{\mu,\nu}(s)=\left[\mu_{h}^{\top}Q_{h}^{\mu,\nu}\nu_{h}\right](s),
\]
\[
Q_{h}^{\mu,\nu}(s,a,b)=r_{h}(s,a,b)+\left[\mathbb{P}_{h}V_{h+1}^{\mu,\nu}\right](s,a,b).
\]

\begin{algorithm}
    \caption{\label{alg:main}Optimistic-follow-the-regularized-leader for solving
    two-player zero-sum Markov games}
    
    \textbf{Input: }Stepsize $\eta$, reward function $r$, probability
    transition function $\mathbb{P}$.
    
    \textbf{Initialization:} $Q_{h}^{0}\equiv0$ for all $h\in[H]$.
    
    \textbf{For iteration $1$ to $T$, do}
    \begin{itemize}
    \item \textbf{Policy Update:} for all state $s\in\mathcal{S}$, horizon
    $h\in[H]$, \begin{subequations}\label{subeq:policy-update}
    \begin{align}
    \mu_{h}^{t}(a & \mid s)\propto\exp\left(\frac{\eta}{w_{t}}\left[\sum_{i=1}^{t-1}w_{i}\left[Q_{h}^{i}\nu_{h}^{i}\right](s,a)+w_{t}\left[Q_{h}^{t-1}\nu_{h}^{t-1}\right](s,a)\right]\right),\label{eq:max-update}\\
    \nu_{h}^{t}(b\mid s) & \propto\exp\left(-\frac{\eta}{w_{t}}\left[\sum_{i=1}^{t-1}w_{i}\left[(Q_{h}^{i})^{\top}\mu_{h}^{i}\right](s,b)+w_{t}\left[(Q_{h}^{t-1})^{\top}\mu_{h}^{t-1}\right](s,b)\right]\right).\label{eq:min-update}
    \end{align}
    \end{subequations}
    \item \textbf{Value Update:} for all $s\in\mathcal{S},a\in\mathcal{A},b\in\mathcal{B}$,
    from $h=H$ to 1,
    \begin{equation}
    Q_{h}^{t}(s,a,b)=(1-\alpha_{t})Q_{h}^{t-1}(s,a,b)+\alpha_{t}\left(r_{h}+\mathbb{P}_{h}\left[(\mu_{h+1}^{t})^{\top}Q_{h+1}^{t}\nu_{h+1}^{t}\right]\right)(s,a,b),\label{eq:Q-update}
    \end{equation}
    \end{itemize}
    \textbf{Output average policy: }for all $s\in\mathcal{S},h\in[H]$
    \begin{equation}
    \hat{\mu}_{h}(\cdot\mid s)\coloneqq\sum_{t=1}^{T}\alpha_{T}^{t}\mu_{h}^{t}(\cdot\mid s),\quad\hat{\nu}_{h}(\cdot\mid s)\coloneqq\sum_{t=1}^{T}\alpha_{T}^{t}\nu_{h}^{t}(\cdot\mid s).\label{eq:output-policy}
    \end{equation}
\end{algorithm}

\paragraph{Optimistic-follow-the-regularized-leader. }

Now we are ready to introduce the optimistic-follow-the-regularized-leader
(OFTRL) algorithm for solving two-player zero-sum Markov games, which
has appeared in the paper by \cite{zhang2022policy}. See Algorithm~\ref{alg:main}
for the full specification. 

In a nutshell, the algorithm has three main components. The first
is the policy update~(\ref{subeq:policy-update}) using \emph{weighted
OFTRL} for both the max and min players. As one can see, compared
to the standard follow-the-regularized-leader algorithm, the \emph{weighted
OFTRL} adds a loss predictor $[Q^{t-1}\nu^{t-1}](s,a)$ and deploys
a weighted update according to the weights $\{w_{i}\}_{1\le i\le t}$,
which we shall define momentarily. The second component is the backward
value update~(\ref{eq:Q-update}) using weighted average of the previous
estimates and the Bellman updates. The last essential part is outputting
a weighted policy~(\ref{eq:output-policy}) over all the historical
policies. As one can realize, weights play a big role in specifying
the OFTRL algorithm. In particular, we set 
\begin{equation}
\alpha_{t}\coloneqq\frac{H+1}{H+t},\quad\alpha_{t}^{t}\coloneqq\alpha_{t},\quad\alpha_{t}^{i}\coloneqq\alpha_{i}\prod_{j=i+1}^{t}(1-\alpha_{j}),\quad w_{i}\coloneqq\frac{\alpha_{t}^{i}}{\alpha_{t}^{1}}=\frac{\alpha_{i}}{\alpha_{1}\prod_{j=2}^{i}(1-\alpha_{j})},\label{eq:alpha-def}
\end{equation}
which are the same choices as in the paper by \cite{zhang2022policy}. 

\section{Main result and overview of the proof}

With the preliminaries in place, we are in a position to state our
main result for OFTRL in two-player zero-sum Markov games. 

\begin{theorem}\label{thm:main}Consider Algorithm~\ref{alg:main}
with $\eta=C_{\eta}H^{-2}$ for some constant $C_{\eta}\le1/8$. The
output policy pair $(\hat{\mu},\hat{\nu})$ satisfies 
\[
\NEgap(\hat{\mu},\hat{\nu})\le\frac{320C_{\eta}^{-1}H^{5}\cdot\log(AB)}{T}.
\]
\end{theorem}

Several remarks on Theorem~\ref{thm:main} are in order. First, Theorem~\ref{thm:main}
demonstrates that OFTRL can find an $O(T^{-1})$-approximate Nash
equilibrium in $T$ iterations. This improves the $\tilde{O}(T^{-5/6})$
rate proved in the prior work~\cite{zhang2022policy}, and also matches
the empirical evidence provided therein. While the paper by \cite{zhang2022policy}
also provides a modified OFTRL algorithm that achieves an $\tilde{O}(T^{-1})$
rate by maintaining two separate value estimators (one for the max-player
and the other for the min-player), the OFTRL algorithm studied herein
is more natural and also computationally simpler. Second, this rate
is nearly unimprovable even in the simpler two-player zero-sum normal-form
games~\cite{daskalakis2011near}. It is also worth
pointing out that algorithms with $\tilde{O}(T^{-1})$ rate have been
proposed in the literature~\cite{cen2021fast,zhao2021provably}.
However, compared to those algorithms, OFTRL does not require one
to approximately solve a normal-form game in each iteration. 
Lastly, Theorem~\ref{thm:main} allows any $C_\eta\in (0,1/8]$ while $C_\eta=1/8$ is optimal for the bound on $\NEgap$.

Before embarking on the formal proof, we would like to immediately
provide an overview of our proof techniques. 

\paragraph*{Step 1: controlling $\NEgap$ using the sum of regrets and estimation
error. }

In the simpler normal-form game (i.e., without any state transition
dynamics as in Markov games), it is well known that $\NEgap$ is controlled
by the sum of the regrets of the two players. This would also be the
case for Markov games if in the policy update~(\ref{subeq:policy-update})
by OFTRL, we use the true $Q$-function $Q_{h}^{\star}$ instead of
the estimate $Q_{h}^{t}$. As a result, intuitively, the $\NEgap$
in Markov games should be controlled by both the sum of the regrets
of the two players and also the estimation error $\|Q_{h}^{t}-Q_{h}^{\star}\|_{\infty}$;
see Lemma~\ref{lemma:NEGap}. 

\paragraph*{Step 2: bounding the sum of regrets. }

Given the extensive literature on regret guarantees for optimistic
algorithms~\cite{anagnostides2022last,anagnostides2022uncoupled,zhang2022policy},
it is relatively easy to control the sum of the regrets to obtain
the desired $O(T^{-1})$ rate; see Lemma~\ref{lemma:regret-bound}.
The key is to exploit the stability in the loss vectors. 

\paragraph*{Step 3: bounding estimation error. }

It then boils down to controlling the estimation error $\|Q_{h}^{t}-Q_{h}^{\star}\|_{\infty}$,
in which our main technical contributions lie. Due to the nature of
the Bellman update~(\ref{eq:Q-update}), it is not hard to obtain
a recursive relation for the estimation error; see the recursion~(\ref{eq:recursion-delta}).
However, the undesirable part is that the estimation error depends
on the \emph{maximal} regret between the two players, instead of the
\emph{sum} of the regrets. This calls for technical innovation. Inspired
by the work of~\cite{anagnostides2022last,anagnostides2022uncoupled}
in normal-form games, we make an important observation that the sum
of the regrets is approximately non-negative. In particular, the sum
is lower bounded by the negative estimation error~$\|Q_{h}^{t}-Q_{h}^{\star}\|_{\infty}$;
see Lemma~\ref{lemma:sum_regret_lb}. This lower bound together with
the upper bound in Step 2 allows us to control the maximal regret
via the estimation error~(\ref{eq:max-regret-bound}), which further
yields a recursive relation~(\ref{eq:delta-recursion}) involving estimation errors only.
Solving the recursion leads to the desired result. 

\section{Proof of Theorem~\ref{thm:main}}

\label{sec:Proof-of-main-Theorem}In this section, we present the
proof of our main result, i.e., Theorem~\ref{thm:main}. We first
define a few useful notations. For each step $h\in[H]$, each state
$s\in\mathcal{S}$, and each iteration $t\in[T]$, we define the state-wise
weighted individual regret as \begin{subequations}

\begin{align}
\mathrm{reg}_{h,1}^{t}(s) & \coloneqq\max_{\mu^{\dagger}\in\Delta_{\mathcal{A}}}\sum_{i=1}^{t}\alpha_{t}^{i}\left\langle \mu^{\dagger}-\mu_{h}^{i},Q_{h}^{i}\nu_{h}^{i}\right\rangle (s),\label{def:regret-max}\\
\mathrm{reg}_{h,2}^{t}(s) & \coloneqq\max_{\nu^{\dagger}\in\Delta_{\mathcal{B}}}\sum_{i=1}^{t}\alpha_{t}^{i}\left\langle \nu_{h}^{i}-\nu^{\dagger},(Q_{h}^{i})^{\top}\mu_{h}^{i}\right\rangle (s).\label{def:regret-min}
\end{align}
\end{subequations}We also define the maximal regret as 
\[
\mathrm{reg}_{h}^{t}\coloneqq\max_{s\in\mathcal{S}}\max_{i=1,2}\left\{ \mathrm{reg}_{h,i}^{t}(s)\right\} ,
\]
that maximizes over the players and the states. In addition, for each
step $h\in[H]$, and each iteration $t\in[T]$, we define the estimation
error of the $Q$-function as 
\[
\delta_{h}^{t}\coloneqq\|Q_{h}^{t}-Q_{h}^{\star}\|_{\infty}.
\]

With these notations in place, we first connect the $\NEgap$ with
the sum of regrets $\mathrm{reg}_{h,1}^{T}(s)+\mathrm{reg}_{h,2}^{T}(s)$
as well as the estimation error $\delta_{h}^{t}$. 

\begin{lemma}\label{lemma:NEGap}

One has 
\begin{align*}
\NEgap(\hat{\mu},\hat{\nu}) & \le2\sum_{h=1}^{H}\left\{ \max_{s}\left\{ \mathrm{reg}_{h,1}^{T}(s)+\mathrm{reg}_{h,2}^{T}(s)\right\} +2\sum_{t=1}^{T}\alpha_{T}^{t}\delta_{h}^{t}\right\} .
\end{align*}

\end{lemma}

\noindent See Section~\ref{sec:NEGap} for the proof of this lemma. 

\medskip

It then boils down to controlling $\max_{s}\left\{ \mathrm{reg}_{h,1}^{T}(s)+\mathrm{reg}_{h,2}^{T}(s)\right\} $
and $\sum_{t=1}^{T}\alpha_{T}^{t}\delta_{h}^{t}$. The following two
lemmas provide such control.

\begin{lemma}\label{lemma:regret-bound} For every $h\in[H]$, every
$s\in\mathcal{S}$, and every iteration $t\in[T]$, one has \begin{subequations}\label{subeq:regret-upper-bound}
\begin{align}
\mathrm{reg}_{h,1}^{t}(s) & \le\frac{2H\cdot(\log A)}{\eta t}+\frac{16\eta H^{3}}{t}+2\eta H^{2}\sum_{i=2}^{t}\alpha_{t}^{i}\|\nu_{h}^{i}(\cdot\mid s)-\nu_{h}^{i-1}(\cdot\mid s)\|_{1}^{2}\label{eq:regret-max}\\
 & \quad-\frac{1}{8\eta}\sum_{i=2}^{t}\alpha_{t}^{i-1}\|\mu_{h}^{i}(\cdot\mid s)-\mu_{h}^{i-1}(\cdot\mid s)\|_{1}^{2};\nonumber \\
\mathrm{reg}_{h,2}^{t}(s) & \le\frac{2H\cdot(\log B)}{\eta t}+\frac{16\eta H^{3}}{t}+2\eta H^{2}\sum_{i=2}^{t}\alpha_{t}^{i}\|\mu_{h}^{i}(\cdot\mid s)-\mu_{h}^{i-1}(\cdot\mid s)\|_{1}^{2}\label{eq:regret-min}\\
 & \quad-\frac{1}{8\eta}\sum_{i=2}^{t}\alpha_{t}^{i-1}\|\nu_{h}^{i}(\cdot\mid s)-\nu_{h}^{i-1}(\cdot\mid s)\|_{1}^{2}.\nonumber 
\end{align}
\end{subequations}As a result, when $\eta=C_{\eta}H^{-2}$ for some
constant $C_{\eta}\le1/8$, one has 
\begin{align}
\max_{s}\left\{ \mathrm{reg}_{h,1}^{t}(s)+\mathrm{reg}_{h,2}^{t}(s)\right\} \leq\frac{3C_{\eta}^{-1}H^{3}\cdot\log(AB)}{t}-4\eta H^{3}\sum_{i=2}^{t}\alpha_{t}^{i} & \Big(\|\mu_{h}^{i}(\cdot\mid s)-\mu_{h}^{i-1}(\cdot\mid s)\|_{1}^{2}\label{eq:sum-regret-bound-with-negative}\\
 & \quad+\|\nu_{h}^{i}(\cdot\mid s)-\nu_{h}^{i-1}(\cdot\mid s)\|_{1}^{2}\Big).\nonumber 
\end{align}
\end{lemma}

\noindent See Section~\ref{sec:regret-bound} for the proof of this
lemma. 

\medskip

\begin{lemma}\label{lemma:delta-t-h-bound}Choosing $\eta=C_{\eta}H^{-2}$
for some constant $C_{\eta}\le1/8$, for all $h\in[H]$ and $t\in[T]$,
we have that
\[
\delta_{h}^{t}\leq\frac{5e^{2}C_{\eta}^{-1}H^{4}\cdot\log(AB)}{t}.
\]
\end{lemma}

\noindent See Section~\ref{sec:delta-t-h} for the proof of this
lemma. 

\medskip

Combine Lemmas~\ref{lemma:regret-bound}-\ref{lemma:delta-t-h-bound}
with Lemma~\ref{lemma:NEGap} to arrive at the desired conclusion
that when $\eta=C_{\eta}H^{-2}$ for some constant $C_{\eta}\le1/8$,
\begin{align*}
\NEgap(\hat{\mu},\hat{\nu}) & \leq2\sum_{h=1}^{H}\left\{ \max_{s}\left\{ \mathrm{reg}_{h,1}^{T}(s)+\mathrm{reg}_{h,2}^{T}(s)\right\} +2\sum_{t=1}^{T}\alpha_{T}^{t}\delta_{h}^{t}\right\} \\
 & \leq2\sum_{h=1}^{H}\left\{ \frac{3C_{\eta}^{-1}H^{3}\cdot\log(AB)}{T}+2\sum_{t=1}^{T}\alpha_{T}^{t}\frac{5e^{2}C_{\eta}^{-1}H^{4}\cdot\log(AB)}{t}\right\} \\
 & \leq2H\cdot\left\{ \frac{3C_{\eta}^{-1}H^{3}\cdot\log(AB)}{T}+\frac{20e^{2}C_{\eta}^{-1}H^{4}\cdot\log(AB)}{T}\right\} \\
 & \le\frac{320C_{\eta}^{-1}H^{5}\cdot\log(AB)}{T},
\end{align*}
where the penultimate inequality uses the following important lemma
we have alluded to before. 

\begin{lemma}\label{lemma:our-lemma} For all $t\geq1$, one has
\begin{align}
\sum_{i=1}^{t}\alpha_{t}^{i}\cdot\frac{1}{i} & \le\left(1+\frac{1}{H}\right)\frac{1}{t}.\label{eq:our-result}
\end{align}
\end{lemma} 

\noindent On the surface, this lemma shaves an extra $\log t$ factor
from a simple average of the sequence $\left\{ 1/i\right\} _{i\le t}$
(cf.~Lemma~A.3 in the paper by~\cite{zhang2022policy}). But more
importantly, it shines in the ensuing proof of Lemma~\ref{lemma:delta-t-h-bound}
by enabling the induction step. See Section~\ref{sec:Proof-of-our-Lemma}
for the proof of Lemma~\ref{lemma:our-lemma}, and see the end of
Section~\ref{sec:delta-t-h} for the comment on the benefit of this
improved result. 

\subsection{Proof of Lemma~\ref{lemma:NEGap} \label{sec:NEGap}}

Invoke Lemma C.1 in the paper~\cite{zhang2022policy} to obtain 
\begin{align*}
\NEgap(\hat{\mu},\hat{\nu}) & =V_{1}^{\dagger,\hat{\nu}}(s_{1})-V_{1}^{\star}(s_{1})+V_{1}^{\star}(s_{1})-V_{1}^{\hat{\mu},\dagger}(s_{1})\\
 & \leq2\sum_{h=1}^{H}\max_{s}\left\{ \max_{\mu^{\dagger},\nu^{\dagger}}\left[\left\langle \mu^{\dagger},Q_{h}^{\star}\hat{\nu}_{h}\right\rangle -\left\langle \nu^{\dagger},Q_{h}^{\star\top}\hat{\mu}_{h}\right\rangle \right](s)\right\} .
\end{align*}
By the definition of the output policy $(\hat{\mu},\hat{\nu})$, one
has
\[
\max_{\mu^{\dagger},\nu^{\dagger}}\left[\left\langle \mu^{\dagger},Q_{h}^{\star}\hat{\nu}_{h}\right\rangle -\left\langle \nu^{\dagger},Q_{h}^{\star\top}\hat{\mu}_{h}\right\rangle \right](s)=\max_{\mu^{\dagger},\nu^{\dagger}}\sum_{t=1}^{T}\alpha_{T}^{t}\left[\left\langle \mu^{\dagger},Q_{h}^{\star}\nu_{h}^{t}\right\rangle -\left\langle \nu^{\dagger},Q_{h}^{\star\top}\mu_{h}^{t}\right\rangle \right](s).
\]
Replacing the true value function $Q_{h}^{\star}$ with the value
estimate $Q_{h}^{t}$ yields 
\[
\max_{\mu^{\dagger},\nu^{\dagger}}\left[\left\langle \mu,Q_{h}^{\star}\hat{\nu}_{h}\right\rangle -\left\langle \nu^{\dagger},(Q_{h}^{\star})^{\top}\hat{\mu}_{h}\right\rangle \right](s)\leq\max_{\mu^{\dagger},\nu^{\dagger}}\sum_{t=1}^{T}\alpha_{T}^{t}\left[\left\langle \mu^{\dagger},Q_{h}^{t}\nu_{h}^{t}\right\rangle -\left\langle \nu^{\dagger},(Q_{h}^{t})^{\top}\mu_{h}^{t}\right\rangle \right](s)+2\sum_{t=1}^{T}\alpha_{T}^{t}\delta_{h}^{t},
\]
where we recall $\delta_{h}^{t}=\|Q_{h}^{t}-Q_{h}^{\star}\|_{\infty}$.
The proof is finished by taking the above three relations together
with the observation that
\[
\mathrm{reg}_{h,1}^{T}(s)+\mathrm{reg}_{h,2}^{T}(s)=\max_{\mu^{\dagger},\nu^{\dagger}}\sum_{t=1}^{T}\alpha_{T}^{t}\left[\left\langle \mu^{\dagger},Q_{h}^{t}\nu_{h}^{t}\right\rangle -\left\langle \nu^{\dagger},(Q_{h}^{t})^{\top}\mu_{h}^{t}\right\rangle \right](s).
\]

\subsection{Proof of Lemma~\ref{lemma:regret-bound} \label{sec:regret-bound}}

We prove the regret bound for the max-player (i.e., bound~(\ref{eq:regret-max})).
The bound~(\ref{eq:regret-min}) for the min-player can be obtained
via symmetry. 

First, we make the observation that, the policy update in Algorithm~\ref{alg:main}
for the max-player is exactly the OFTRL algorithm (i.e., Algorithm
4 in the paper \cite{zhang2022policy}) with the loss vector $g_{t}=w_{t}[Q_{h}^{t}\nu_{h}^{t}](s,\cdot)$,
the recency bias $M_{t}=w_{t}[Q_{h}^{t-1}\nu_{h}^{t-1}](s,\cdot)$,
and a learning rate $\eta_{t}=\eta/w_{t}$. Therefore, we can apply
Lemma B.3 from \cite{zhang2022policy} to obtain

\begin{align}
\mathrm{reg}_{h,1}^{t}(s) & =\max_{\mu^{\dagger}}\sum_{i=1}^{t}\alpha_{t}^{i}\left\langle \left(\mu^{\dagger}-\mu_{h}^{i}\right),Q_{h}^{i}\nu_{h}^{i}\right\rangle (s)\nonumber \\
 & =\alpha_{t}^{1}\max_{\mu^{\dagger}}\sum_{i=1}^{t}w_{i}\left\langle \left(\mu^{\dagger}-\mu_{h}^{i}\right),Q_{h}^{i}\nu_{h}^{i}\right\rangle (s)\nonumber \\
 & \le\frac{\alpha_{t}\cdot(\log A)}{\eta}+\underbrace{\alpha_{t}^{1}\sum_{i=1}^{t}\frac{\eta}{w_{i}}\left\Vert \left[w_{i}Q_{h}^{i}\nu_{h}^{i}-w_{i}Q_{h}^{i-1}\nu_{h}^{i-1}\right](s,\cdot)\right\Vert _{\infty}^{2}}_{\eqqcolon\mathrm{Err}_{1}}\label{eq:master-bound}\\
 & \quad-\underbrace{\alpha_{t}^{1}\sum_{i=2}^{t}\frac{w_{i-1}}{8\eta}\|\mu_{h}^{i}(\cdot\mid s)-\mu_{h}^{i-1}(\cdot\mid s)\|_{1}^{2}}_{\eqqcolon\mathrm{Err_{2}}},
\end{align}
where we have used the fact that $w_{i}=\alpha_{t}^{i}/\alpha_{t}^{1}$.
We now move on to bound the term $\mathrm{Err}_{1}$. Use $(a+b)^{2}\leq2a^{2}+2b^{2}$
to see that 
\begin{align*}
\left\Vert \left[Q_{h}^{i}\nu_{h}^{i}-Q_{h}^{i-1}\nu_{h}^{i-1}\right](s,\cdot)\right\Vert _{\infty}^{2} & \le2\left\Vert \left[Q_{h}^{i}\nu_{h}^{i}-Q_{h}^{i-1}\nu_{h}^{i}\right](s,\cdot)\right\Vert _{\infty}^{2}+2\left\Vert \left[Q_{h}^{i-1}\nu_{h}^{i}-Q_{h}^{i-1}\nu_{h}^{i-1}\right](s,\cdot)\right\Vert _{\infty}^{2}\\
 & \le2\|Q_{h}^{i}-Q_{h}^{i-1}\|_{\infty}^{2}+2H^{2}\|\nu_{h}^{i}(\cdot\mid s)-\nu_{h}^{i-1}(\cdot\mid s)\|_{1}^{2},
\end{align*}
where the second line uses Holder's inequality and the fact that $\|Q_{h}^{i-1}\|_{\infty}\leq H$.
In view of the update rule~(\ref{eq:Q-update}) for the $Q$-function,
we further have
\begin{align*}
\|Q_{h}^{i}-Q_{h}^{i-1}\|_{\infty} & =\left\Vert -\alpha_{i}Q_{h}^{i-1}+\alpha_{i}\left(r_{h}+\mathbb{P}_{h}\left[(\mu_{h+1}^{i})^{\top}Q_{h+1}^{i}\nu_{h+1}^{i}\right]\right)\right\Vert _{\infty}\\
 & \le\alpha_{i}\max\left\{ \left\Vert Q_{h}^{i-1}\right\Vert _{\infty},\left\Vert r_{h}+\mathbb{P}_{h}\left[(\mu_{h+1}^{i})^{\top}Q_{h+1}^{i}\nu_{h+1}^{i}\right]\right\Vert _{\infty}\right\} \\
 & \le\alpha_{i}H.
\end{align*}
As a result, we arrive at the bound 
\begin{align*}
\mathrm{Err_{1}} & \le2\eta\alpha_{t}^{1}\sum_{i=1}^{t}w_{i}\left(\alpha_{i}^{2}H^{2}+H^{2}\|\nu_{h}^{i}(\cdot\mid s)-\nu_{h}^{i-1}(\cdot\mid s)\|_{1}^{2}\right)\\
 & =2\eta H^{2}\sum_{i=1}^{t}\alpha_{t}^{i}\alpha_{i}^{2}+2\eta H^{2}\sum_{i=1}^{t}\alpha_{t}^{i}\|\nu_{h}^{i}(\cdot\mid s)-\nu_{h}^{i-1}(\cdot\mid s)\|_{1}^{2},
\end{align*}
where we again use the relation $w_{i}=\alpha_{t}^{i}/\alpha_{t}^{1}$.
Since $\{\alpha_{i}\}_{i\le t}$ is decreasing in $i$, we can apply
Property 6 in Lemma~\ref{lemma:alpha-properties} to obtain 
\[
\sum_{i=1}^{t}\alpha_{t}^{i}\alpha_{i}^{2}\leq\frac{1}{t}\sum_{i=1}^{t}\alpha_{i}^{2}\leq\frac{H+2}{t}\leq\frac{3H}{t},
\]
where the second inequality follows from Property 5 in Lemma~\ref{lemma:alpha-properties}.
In all, we see that 
\begin{equation}
\mathrm{Err}_{1}\leq\frac{6\eta H^{3}}{t}+2\eta H^{2}\sum_{i=1}^{t}\alpha_{t}^{i}\left\Vert \nu_{h}^{i}(\cdot\mid s)-\nu_{h}^{i-1}(\cdot\mid s)\right\Vert _{1}^{2}.\label{eq:Err-1}
\end{equation}
Substitute the upper bound~(\ref{eq:Err-1}) for $\mathrm{Err}_{1}$
into the master bound~(\ref{eq:master-bound}) to obtain 
\begin{align*}
\mathrm{reg}_{h,1}^{t}(s) & \le\frac{\alpha_{t}\cdot(\log A)}{\eta}+\mathrm{Err_{1}}-\mathrm{Err_{2}}\\
 & \le\frac{2H\cdot(\log A)}{\eta t}+\frac{6\eta H^{3}}{t}+2\eta H^{2}\sum_{i=1}^{t}\alpha_{t}^{i}\|\nu_{h}^{i}(\cdot\mid s)-\nu_{h}^{i-1}(\cdot\mid s)\|_{1}^{2}\\
 & \quad-\frac{1}{8\eta}\sum_{i=2}^{t}\alpha_{t}^{i-1}\|\mu_{h}^{i}(\cdot\mid s)-\mu_{h}^{i-1}(\cdot\mid s)\|_{1}^{2},
\end{align*}
where in the first inequality we use $\alpha_{t}=(H+1)/(H+t)\le2H/t$.
Since $\|\nu_{h}^{i}(\cdot\mid s)-\nu_{h}^{i-1}(\cdot\mid s)\|_{1}\le2$
and $\alpha_{t}^{1}\le1/t$ (see Property 2 of Lemma~\ref{lemma:alpha-properties}),
we can take the term $i=1$ out and reach
\[
\mathrm{reg}_{h,1}^{t}(s)\le\frac{2H\cdot(\log A)}{\eta t}+\frac{16\eta H^{3}}{t}+2\eta H^{2}\sum_{i=2}^{t}\alpha_{t}^{i}\|\nu_{h}^{i}(\cdot\mid s)-\nu_{h}^{i-1}(\cdot\mid s)\|_{1}^{2}-\frac{1}{8\eta}\sum_{i=2}^{t}\alpha_{t}^{i-1}\|\mu_{h}^{i}(\cdot\mid s)-\mu_{h}^{i-1}(\cdot\mid s)\|_{1}^{2}.
\]
This finishes the proof of the regret bound~(\ref{eq:regret-max})
for the max-player. The bound~(\ref{eq:regret-min}) for the min-player
can be obtained via symmetry. 

Combine the two bounds~(\ref{eq:regret-max}) and (\ref{eq:regret-min})
see that 
\begin{align}
\mathrm{reg}_{h,1}^{t}(s)+\mathrm{reg}_{h,2}^{t}(s) & \le\frac{2H\cdot\log(AB)}{\eta t}+\frac{32\eta H^{3}}{t}\nonumber \\
 & \quad+\sum_{i=2}^{t}\left(2\eta H^{2}\alpha_{t}^{i}-\frac{\alpha_{t}^{i-1}}{8\eta}\right)\left(\|\mu_{h}^{i}(\cdot\mid s)-\mu_{h}^{i-1}(\cdot\mid s)\|_{1}^{2}+\|\nu_{h}^{i}(\cdot\mid s)-\nu_{h}^{i-1}(\cdot\mid s)\|_{1}^{2}\right).\label{eq:sum-regret-bound-raw}
\end{align}
When $\eta\le1/(8H^{2})$, one has 
\[
2\eta H^{2}\alpha_{t}^{i}-\frac{\alpha_{t}^{i-1}}{8\eta}\le2\eta H^{3}\alpha_{t}^{i}-\frac{\alpha_{t}^{i-1}}{8\eta}\le-4\eta H^{3}\alpha_{t}^{i},
\]
where we have used Property 3 of Lemma~\ref{lemma:alpha-properties},
i.e., $\alpha_{t}^{i-1}/\alpha_{t}^{i}\ge1/H$. Consequently, with
$\eta=C_{\eta}H^{-2}$ for some constant $C_{\eta}\le1/8$, the bound~(\ref{eq:sum-regret-bound-raw})
reads 
\begin{align*}
\max_{s}\left\{ \mathrm{reg}_{h,1}^{t}(s)+\mathrm{reg}_{h,2}^{t}(s)\right\} \leq\frac{3C_{\eta}^{-1}H^{3}\cdot\log(AB)}{t}-4\eta H^{3}\sum_{i=2}^{t}\alpha_{t}^{i} & \Big(\|\mu_{h}^{i}(\cdot\mid s)-\mu_{h}^{i-1}(\cdot\mid s)\|_{1}^{2}\\
 & \quad+\|\nu_{h}^{i}(\cdot\mid s)-\nu_{h}^{i-1}(\cdot\mid s)\|_{1}^{2}\Big),
\end{align*}
where we assume the choice of players is non-trivial, i.e., $AB\ge2$. 

\subsection{Proof of Lemma~\ref{lemma:delta-t-h-bound} \label{sec:delta-t-h}}

By Lemma C.2 in the paper~\cite{zhang2022policy}, for any $h\in[H-1]$,
we have the recursive relation 
\begin{equation}
\delta_{h}^{t}\le\sum_{i=1}^{t}\alpha_{t}^{i}\delta_{h+1}^{i}+\mathrm{reg}_{h+1}^{t},\label{eq:recursion-delta}
\end{equation}
where we recall $\mathrm{reg}_{h+1}^{t}=\max_{s}\max_{i=1,2}\{\mathrm{reg}_{h+1,i}^{t}(s)\}$. 

\paragraph{Step 1: Bounding $\mathrm{reg}_{h+1}^{t}$. }

In view of this recursion~(\ref{eq:recursion-delta}), one needs
to control the maximal regret $\mathrm{reg}_{h+1}^{t}$ over the two
players. Lemma~\ref{lemma:regret-bound} provides us with precise
control of the individual regrets $\mathrm{reg}_{h,1}^{t}(s)$ and
$\mathrm{reg}_{h,2}^{t}(s)$: \begin{subequations}
\begin{align}
\mathrm{reg}_{h,1}^{t}(s) & \le\frac{3C_{\eta}^{-1}H^{3}\cdot(\log AB)}{t}+2\eta H^{2}\sum_{i=2}^{t}\alpha_{t}^{i}\|\nu_{h}^{i}(\cdot\mid s)-\nu_{h}^{i-1}(\cdot\mid s)\|_{1}^{2},\label{eq:reg_mu_ub}\\
\mathrm{reg}_{h,2}^{t}(s) & \le\frac{3C_{\eta}^{-1}H^{3}\cdot(\log AB)}{t}+2\eta H^{2}\sum_{i=2}^{t}\alpha_{t}^{i}\|\mu_{h}^{i}(\cdot\mid s)-\mu_{h}^{i-1}(\cdot\mid s)\|_{1}^{2},\label{eq:reg_nu_ub}
\end{align}
\end{subequations}where we have substituted $\eta=C_{\eta}H^{-2}$
for $C_{\eta}\le1/8$ and $AB\ge2$. We have also ignored the negative
terms on the right hand sides of (\ref{eq:regret-max}) and (\ref{eq:regret-min}).
Therefore, to control individual regrets, it suffices to bound the
second-order path lengths $2\eta H^{2}\sum_{i=2}^{t}\alpha_{t}^{i}\|\mu_{h}^{i}(\cdot\mid s)-\mu_{h}^{i-1}(\cdot\mid s)\|_{1}^{2}$
and $2\eta H^{2}\sum_{i=2}^{t}\alpha_{t}^{i}\|\nu_{h}^{i}(\cdot\mid s)-\nu_{h}^{i-1}(\cdot\mid s)\|_{1}^{2}$.
To this end, the following lemma proves crucial, whose proof is deferred
to the end of this section. 

\begin{lemma}\label{lemma:sum_regret_lb}For each $t,h$ and $s$,
one has
\[
\mathrm{reg}_{h,1}^{t}(s)+\mathrm{reg}_{h,2}^{t}(s)\ge-2\sum_{i=1}^{t}\alpha_{t}^{i}\delta_{h}^{i}.
\]
\end{lemma}

In words, Lemma~\ref{lemma:sum_regret_lb} reveals the approximate
non-negativity of the sum of the regrets. This together with the upper
bound~(\ref{eq:sum-regret-bound-with-negative}) in Lemma~\ref{lemma:regret-bound}
implies 
\[
2\eta H^{2}\sum_{i=2}^{t}\left(\alpha_{t}^{i}\|\mu_{h}^{i}(\cdot\mid s)-\mu_{h}^{i-1}(\cdot\mid s)\|_{1}^{2}+\|\nu_{h}^{i}(\cdot\mid s)-\nu_{h}^{i-1}(\cdot\mid s)\|_{1}^{2}\right)\le\frac{3C_{\eta}^{-1}H^{2}\cdot\log(AB)}{2t}+\frac{1}{H}\sum_{i=1}^{t}\alpha_{t}^{i}\delta_{h}^{i}.
\]
Feeding this back to (\ref{eq:reg_mu_ub}) and (\ref{eq:reg_nu_ub}),
we obtain 
\begin{equation}
\mathrm{reg}_{h}^{t}=\max_{s}\max_{i=1,2}\left\{ \mathrm{reg}_{h,i}^{t}(s)\right\} \le\frac{5C_{\eta}^{-1}H^{3}\cdot\log(AB)}{t}+\frac{1}{H}\sum_{i=1}^{t}\alpha_{t}^{i}\delta_{h}^{i}.\label{eq:max-regret-bound}
\end{equation}

\paragraph{Step 2: Bounding $\delta_{h}^{t}$. }

Substituting the maximal regret bound~(\ref{eq:max-regret-bound})
into the recursion~(\ref{eq:recursion-delta}), we arrive at
\begin{equation}
\delta_{h}^{t}\le\left(1+\frac{1}{H}\right)\sum_{i=1}^{t}\alpha_{t}^{i}\delta_{h+1}^{i}+\frac{5C_{\eta}^{-1}H^{3}\cdot\log(AB)}{t}.\label{eq:delta-recursion}
\end{equation}
We continue the proof of Lemma~\ref{lemma:delta-t-h-bound} via induction
on $h$. More precisely, we aim to inductively establish the claim
\begin{equation}
\delta_{h}^{t}\le\sum_{h'=h}^{H}\left(1+\frac{1}{H}\right)^{2(H-h')}\cdot\frac{5C_{\eta}^{-1}H^{3}\cdot\log(AB)}{t}.\label{eq:delta-induction-hypothesis}
\end{equation}
First note that the induction hypothesis holds naturally for $h=H$
as $\delta_{H}^{t}=0$ for all $1\leq t\leq T$. Now assume that the
induction hypothesis is true for some $2\leq h+1\leq H$ and for all
$1\leq t\leq T$. Our goal is to show that~(\ref{eq:delta-induction-hypothesis})
continues to hold for the previous step~$h$ and for all $1\leq t\leq T$.
By the recursion~(\ref{eq:delta-recursion}) and the induction hypothesis,
one has for any $1\leq t\leq T$: 
\begin{align*}
\delta_{h}^{t} & \le\left(1+\frac{1}{H}\right)\sum_{i=1}^{t}\alpha_{t}^{i}\delta_{h+1}^{i}+\frac{5C_{\eta}^{-1}H^{3}\cdot\log(AB)}{t}\\
 & \le\left(1+\frac{1}{H}\right)\sum_{i=1}^{t}\alpha_{t}^{i}\left(\sum_{h'=h+1}^{H}\left(1+\frac{1}{H}\right)^{2(H-h')}\cdot\frac{5C_{\eta}^{-1}H^{3}\cdot\log(AB)}{t}\right)+\frac{5C_{\eta}^{-1}H^{3}\cdot\log(AB)}{t}.
\end{align*}
Apply Lemma~\ref{lemma:our-lemma} to obtain 
\[
\sum_{i=1}^{t}\alpha_{t}^{i}\cdot\frac{5C_{\eta}^{-1}H^{3}\cdot\log(AB)}{i}\le\left(1+\frac{1}{H}\right)\frac{5C_{\eta}^{-1}H^{3}\cdot\log(AB)}{t}.
\]
This leads to the conclusion that 
\begin{align*}
\delta_{h}^{t} & \le\left(1+\frac{1}{H}\right)\sum_{h'=h+1}^{H}\left(1+\frac{1}{H}\right)^{2(H-h')}\left(1+\frac{1}{H}\right)\frac{5C_{\eta}^{-1}H^{3}\cdot\log(AB)}{t}+\frac{5C_{\eta}^{-1}H^{3}\cdot\log(AB)}{t}\\
 & =\sum_{h'=h+1}^{H}\left(1+\frac{1}{H}\right)^{2(H-h'+1)}\frac{5C_{\eta}^{-1}H^{3}\cdot\log(AB)}{t}+\frac{5C_{\eta}^{-1}H^{3}\cdot\log(AB)}{t}\\
 & =\sum_{h'=h}^{H}\left(1+\frac{1}{H}\right)^{2(H-h')}\frac{5C_{\eta}^{-1}H^{3}\cdot\log(AB)}{t}.
\end{align*}
This finishes the induction. 

This bound on $\delta_{h}^{t}$ can be further simplified by 
\begin{align*}
\delta_{h}^{t} & \le\sum_{h'=h}^{H}\left(1+\frac{1}{H}\right)^{2(H-h')}\cdot\frac{5C_{\eta}^{-1}H^{3}\cdot\log(AB)}{t}\\
 & \le H\left(1+\frac{1}{H}\right)^{2H}\cdot\frac{5C_{\eta}^{-1}H^{3}\cdot\log(AB)}{t}\\
 & \le\frac{5e^{2}C_{\eta}^{-1}H^{4}\cdot\log(AB)}{t}.
\end{align*}
 This finishes the proof, and we are left with proving Lemma~\ref{lemma:sum_regret_lb}. 

\paragraph{Proof of Lemma~\ref{lemma:sum_regret_lb}.}

Recall that
\[
\mathrm{reg}_{h,1}^{t}(s)+\mathrm{reg}_{h,2}^{t}(s)=\max_{\mu^{\dagger},\nu^{\dagger}}\sum_{i=1}^{t}\alpha_{t}^{i}\left[\left\langle \mu^{\dagger},Q_{h}^{i}\nu_{h}^{i}\right\rangle -\left\langle \nu^{\dagger},(Q_{h}^{i})^{\top}\mu_{h}^{i}\right\rangle \right](s).
\]
Replace the estimation $Q_{h}^{i}$ with $Q_{h}^{\star}$ to obtain
\begin{align*}
\mathrm{reg}_{h,1}^{t}(s)+\mathrm{reg}_{h,2}^{t}(s) & \ge\max_{\mu^{\dagger},\nu^{\dagger}}\left[\sum_{i=1}^{t}\alpha_{t}^{i}\left[\left\langle \mu^{\dagger},Q_{h}^{\star}\nu_{h}^{i}\right\rangle -\left\langle \nu^{\dagger},(Q_{h}^{\star})^{\top}\mu_{h}^{i}\right\rangle \right](s)\right.\\
 & \hfill\qquad\qquad\left.+\sum_{i=1}^{t}\alpha_{t}^{i}\left[\left\langle \mu^{\dagger},\left(Q_{h}^{i}-Q_{h}^{\star}\right)\nu_{h}^{i}\right\rangle -\left\langle \nu^{\dagger},\left(Q_{h}^{i}-Q_{h}^{\star}\right){}^{\top}\mu_{h}^{i}\right\rangle \right](s)\right].
\end{align*}
Lower bounding the term involving $Q_{h}^{i}-Q_{h}^{\star}$ yields
\[
\mathrm{reg}_{h,1}^{t}(s)+\mathrm{reg}_{h,2}^{t}(s)\ge\max_{\mu^{\dagger},\nu^{\dagger}}\left[\sum_{i=1}^{t}\alpha_{t}^{i}\left[\left\langle \mu^{\dagger},Q_{h}^{\star}\nu_{h}^{i}\right\rangle -\left\langle \nu^{\dagger},(Q_{h}^{\star})^{\top}\mu_{h}^{i}\right\rangle \right](s)\right]-2\sum_{i=1}^{t}\alpha_{t}^{i}\delta_{h}^{i}.
\]
where recall $\delta_{h}^{i}=\|Q_{h}^{i}-Q_{h}^{\star}\|_{\infty}$.
Now observe that $\sum_{i=1}^{t}\alpha_{t}^{i}\mu_{h}^{i}(\cdot\mid s)$
and $\sum_{i=1}^{t}\alpha_{t}^{i}\nu_{h}^{i}(\cdot\mid s)$ are valid
policies, which implies 
\begin{align*}
 & \max_{\mu^{\dagger},\nu^{\dagger}}\left[\sum_{i=1}^{t}\alpha_{t}^{i}\left[\left\langle \mu^{\dagger},Q_{h}^{\star}\nu_{h}^{i}\right\rangle -\left\langle \nu^{\dagger},(Q_{h}^{\star})^{\top}\mu_{h}^{i}\right\rangle \right](s)\right]\\
 & \quad=\max_{\mu^{\dagger},\nu^{\dagger}}\left[\left\langle \mu^{\dagger},Q_{h}^{\star}\left(\sum_{i=1}^{t}\alpha_{t}^{i}\nu_{h}^{i}\right)\right\rangle (s)-\left\langle \nu^{\dagger},Q_{h}^{\star\top}\left(\sum_{i=1}^{t}\alpha_{t}^{i}\mu_{h}^{i}\right)\right\rangle (s)\right]\\
 & \quad\ge\left\langle \left(\sum_{i=1}^{t}\alpha_{t}^{i}\mu_{h}^{i}\right),Q_{h}^{\star}\left(\sum_{i=1}^{t}\alpha_{t}^{i}\nu_{h}^{i}\right)\right\rangle (s)-\left\langle \left(\sum_{i=1}^{t}\alpha_{t}^{i}\nu_{h}^{i}\right),Q_{h}^{\star\top}\left(\sum_{i=1}^{t}\alpha_{t}^{i}\mu_{h}^{i}\right)\right\rangle (s)\\
 & \quad=0.
\end{align*}
Combine the above two inequalities to finish the proof. 

In the end, it is worth pointing out that without the improved inequality
in Lemma~\ref{lemma:our-lemma}, one would necessarily incur an extra
$\log T$ factor in each induction step. Consequently, the recursion
will fail due to the explosion at a rate of $(\log T)^{H}$. 

\subsection{Proof of Lemma~\ref{lemma:our-lemma}\label{sec:Proof-of-our-Lemma}}

We prove the claim via induction. The base case $t=1$ is true since
$\alpha_{1}^{1}\cdot1=1\le1+1/H$. Now assume that the inequality~(\ref{eq:our-result})
holds for some $t\geq1$, and we aim to prove that it continues to
hold at $t+1$. We first make the observation that for all $i\le t$
\begin{align*}
\alpha_{t+1}^{i} & =\alpha_{i}\prod_{j=i+1}^{t+1}(1-\alpha_{j})=(1-\alpha_{t+1})\alpha_{i}\prod_{j=i+1}^{t}(1-\alpha_{j})=(1-\alpha_{t+1})\alpha_{t}^{i}.
\end{align*}
This allows us to rewrite $\sum_{i=1}^{t+1}\alpha_{t+1}^{i}\cdot\frac{1}{i}$
as 
\begin{align*}
\sum_{i=1}^{t+1}\alpha_{t+1}^{i}\cdot\frac{1}{i} & =(1-\alpha_{t+1})\left(\sum_{i=1}^{t}\alpha_{t}^{i}\cdot\frac{1}{i}\right)+\alpha_{t+1}\cdot\frac{1}{t+1}\\
 & \leq(1-\alpha_{t+1})\left(1+\frac{1}{H}\right)\frac{1}{t}+\frac{\alpha_{t+1}}{t+1},
\end{align*}
where the second line follows from the induction hypothesis. Note
that $\alpha_{t+1}=\frac{H+1}{H+t+1}$. We can continue the derivation
as 
\begin{align*}
\sum_{i=1}^{t+1}\alpha_{t+1}^{i}\cdot\frac{1}{i} & \leq\left(1+\frac{1}{H}\right)\frac{t}{H+t+1}\cdot\frac{1}{t}+\frac{H+1}{H+t+1}\cdot\frac{1}{t+1}\\
 & =\left(1+\frac{1}{H}\right)\frac{t+1}{H+t+1}\cdot\frac{1}{t+1}+\left(1+\frac{1}{H}\right)\frac{H}{H+t+1}\cdot\frac{1}{t+1}\\
 & =\left(1+\frac{1}{H}\right)\frac{1}{t+1}.
\end{align*}
This finishes the proof. 

\section{Discussion}

In this paper, we prove that the optimistic-follow-the-regularized-leader
algorithm, together with smooth value updates, converges to an $O(T^{-1})$-approximate
Nash equilibrium in two-player zero-sum Markov games. This improves
the $\tilde{O}(T^{-5/6})$ rate proved in the paper~\cite{zhang2022policy}.
Quite a few interesting directions are open. Below we single out a
few of them. First, although our rate is unimprovable in the dependence
on $T$, it is likely sub-optimal in its dependence on the horizon
$H$. Improving such dependence and proving any sort of lower bound
on it are both interesting and important for finite-horizon Markov
games. Second, we focus on the simple two-player zero-sum games. It
is an important open question to see whether one can generalize the
proof technique herein to the multi-player general-sum Markov games
and to other solution concepts in games (e.g., coarse correlated equilibria,
and correlated equilibria). 

\bibliographystyle{alpha}
\bibliography{All-of-Bibs}

\appendix

\section{Properties of $\alpha_{t}^{i}$ \label{sec:Properties-alpha}}

This section collects a few useful properties of the sequences $\left\{ \alpha_{t}\right\} _{t\geq1}$
and $\{\alpha_{t}^{i}\}_{t\geq1,1\leq i\leq t}$. Some of these results
have appeared in prior work~\cite{jin2018q,zhang2022policy}. For
completeness, we include all the proofs here.

To help reading, we repeat the definitions below: for each $t\geq1$,
and $1\leq i\leq t$, we define \begin{subequations}

\begin{align}
\alpha_{t} & =\alpha_{t}^{t}=\frac{H+1}{H+t},\qquad\text{and}\label{eq:def-alpha-1}\\
\alpha_{t}^{i} & =\alpha_{i}\prod_{j=i+1}^{t}(1-\alpha_{j}).\label{eq:def-alpha-2}
\end{align}
\end{subequations}

\begin{lemma}\label{lemma:alpha-properties}Fix any $t\geq1$. The
following properties are true: 
\begin{enumerate}
\item The sequence $\{\alpha_{t}^{i}\}_{1\le i\le t}$ sums to 1, i.e.,
$\sum_{i=1}^{t}\alpha_{t}^{i}=1.$
\item For all $1\leq i\leq t$, one has $\alpha_{t}^{i}\le i/t$. 
\item For the relative weight defined by $w_{i}=\alpha_{t}^{i}/\alpha_{t}^{1}$
(note that this is the same for every $t\ge i$), we have
\[
\frac{w_{i}}{w_{i-1}}=\frac{\alpha_{t}^{i}}{\alpha_{t}^{i-1}}=\frac{H+i-1}{i-1}\le H.
\]
\item The sequence $\{\alpha_{t}^{i}\}_{1\le i\le t}$ is increasing in
$i$. 
\item On the sum of squares of the weights, we have 
\[
\sum_{i=1}^{t}(\alpha_{t}^{i})^{2}\le\sum_{i=1}^{t}\alpha_{i}^{2}\le H+2.
\]
\item For any non-increasing sequence $\left\{ b_{i}\right\} _{1\leq i\leq t}$,
one has
\begin{align*}
\sum_{i=1}^{t}\alpha_{t}^{i}b_{i} & \le\frac{1}{t}\sum_{i=1}^{t}b_{i}.
\end{align*}
\end{enumerate}
\end{lemma}

\begin{proof}Property 1 follows directly from the definitions of
$\left\{ \alpha_{t}^{i}\right\} _{1\le i\le t}$. 

Now we move on to Property 2. It trivially holds for $i=t$. Therefore
we focus on the case when $1\leq i\leq t-1$. By definition, we have 

\begin{align}
\alpha_{t}^{i} & =\alpha_{i}\prod_{j=i+1}^{t}(1-\alpha_{j})\le\prod_{j=i+1}^{t}(1-\alpha_{j})=\prod_{j=i+1}^{t}\frac{j-1}{H+j}.\label{eq:prop-2}
\end{align}
where the inequality holds since $\alpha_{i}\leq1$ for all $1\leq i\leq t$,
and the last relation is the definition of $\alpha_{j}$. Expanding
the right hand side of~(\ref{eq:prop-2}), we have 
\[
\alpha_{t}^{i}\leq\frac{i}{H+i+1}\times\frac{i+1}{H+i+2}\times\cdots\times\frac{t-1}{H+t}\leq\frac{i}{H+t},
\]
where we only keep the first numerator and the last denominator. Property
2 then follows.

Property 3 is trivial. Hence we omit the proof. In addition, Property
3 implies Property 4 since $\frac{\alpha_{t}^{i}}{\alpha_{t}^{i-1}}=\frac{H+i-1}{i-1}\geq1$. 

For Property 5, the first inequality holds since $0\leq\alpha_{i}\leq1$
for all $1\leq i\leq t$. For the second inequality, one has 

\begin{align*}
\sum_{i=1}^{t}\alpha_{i}^{2} & =1+\sum_{i=2}^{t}\left(\frac{H+1}{H+i}\right)^{2}\le1+(H+1)^{2}\sum_{i=2}^{t}\left(\frac{1}{(H+i-1)(H+i)}\right).
\end{align*}
Expanding this as a telescoping sum, we see that 

\begin{align*}
\sum_{i=1}^{t}\alpha_{i}^{2} & \le1+(H+1)^{2}\sum_{i=2}^{t}\left(\frac{1}{H+i-1}-\frac{1}{H+i}\right)\\
 & \le1+(H+1)^{2}\frac{1}{H+1}\\
 & =H+2.
\end{align*}

Lastly, for Property 6, we have 
\[
\sum_{i=1}^{t}\alpha_{t}^{i}b_{i}-\frac{1}{t}\sum_{i=1}^{t}b_{i}=\sum_{i=1}^{t}(\alpha_{t}^{i}-\frac{1}{t})b_{i}.
\]
Let $i_{0}\coloneqq\sup_{i}\left\{ \alpha_{t}^{i}\le1/t\right\} $.
Since $\{\alpha_{t}^{i}\}$ is increasing in $i$ (cf.~Property 4)
and $\sum_{i=1}^{t}\alpha_{t}^{i}=1$ (cf.~Property 1), we know that
$i_{0}$ is well defined, i.e., $1\leq i_{0}\leq t$. Since $\left\{ \alpha_{t}^{i}\right\} _{i\le t}$
(resp.~$\{b_{i}\}_{i\le t}$) is increasing (resp.~non-increasing),
we have $\alpha_{t}^{i}\le1/t$ and $b_{i}\ge b_{i_{0}}$ for all
$i\leq i_{0}$. As a result, we obtain $(\alpha_{t}^{i}-1/t)b_{i}\le(\alpha_{t}^{i}-1/t)b_{i_{0}}$
for all $i\le i_{0}$. Similarly, one has $\alpha_{t}^{i}>1/t$ and
$b_{i}\le b_{i_{0}}$ for all $i>i_{0}$, which implies $(\alpha_{t}^{i}-1/t)b_{i}\le(\alpha_{t}^{i}-1/t)b_{i_{0}}$
for all $i>i_{0}$. Take these two relations together to see that 

\begin{align*}
\sum_{i=1}^{t}(\alpha_{t}^{i}-1/t)b_{i} & \le\sum_{i=1}^{t}(\alpha_{t}^{i}-1/t)b_{i_{0}}=0,
\end{align*}
where the last equality uses the fact from Property 1, namely $\sum_{i=1}^{t}\alpha_{t}^{i}=1$. 

\end{proof}

\end{document}